\documentclass[letterpaper]{article} 
\usepackage{aaai2026}  
\usepackage{times}     
\usepackage{helvet}    
\usepackage{courier}   
\usepackage[hyphens]{url}  
\usepackage{graphicx}  
\usepackage{natbib}    
\usepackage{caption}   
\usepackage{nameref}
\usepackage{amsthm}
\newtheorem{property}{Property}

\usepackage{algorithm}
\usepackage{algorithmic}
\usepackage{multirow}
\usepackage[table,xcdraw]{xcolor}
\usepackage{newfloat}
\usepackage{listings}
\usepackage{makecell}
\usepackage{tabularx}
\usepackage{booktabs}
\usepackage{array}
\usepackage{amsmath, amsfonts, bm}
\usepackage{nicefrac}
\usepackage{microtype}
\usepackage{fancyhdr}
\usepackage{lipsum}

\definecolor{corabg}{RGB}{240,248,255} 
\definecolor{citebg}{RGB}{255,240,245} 
\definecolor{pubmbg}{RGB}{240,255,240} 

\newtheorem{lemma}{Lemma}
\newtheorem{theorem}{Theorem}
\newtheorem{corollary}{Corollary}

\newtheorem{definition}{Definition}
\newtheorem{remark}{Remark}

\DeclareCaptionStyle{ruled}{labelfont=normalfont,labelsep=colon,strut=off} 
\floatstyle{ruled}
\newfloat{listing}{tb}{lst}{}
\floatname{listing}{Listing}

\title{Adaptive Initial Residual Connections for GNNs with Theoretical Guarantees}

\author{
    Mohammad Shirzadi, 
    Ali Safarpoor Dehkordi, 
    Ahad N. Zehmakan
}
\affiliations {
    School of Computing, Australian National University, Canberra, Australia\\
    \{mohammad.shirzadi, ali.safarpoordehkordi, ahadn.zehmakan\}@anu.edu.au
}

\usepackage{bibentry}

\begin{document}

\maketitle

\begin{abstract}
Message passing is the core operation in graph neural networks, where each node updates its embeddings by aggregating information from its neighbors. However, in deep architectures, this process often leads to diminished expressiveness. A popular solution is to use residual connections, where the input from the current (or initial) layer is added to aggregated neighbor information to preserve embeddings across layers. Following a recent line of research, we investigate an adaptive residual scheme in which different nodes have varying residual strengths. We prove that this approach prevents oversmoothing; particularly, we show that the Dirichlet energy of the embeddings remains bounded away from zero. This is the first theoretical guarantee not only for the adaptive setting, but also for static residual connections (where residual strengths are shared across nodes) with activation functions. Furthermore, extensive experiments show that this adaptive approach outperforms standard and state-of-the-art message passing mechanisms, especially on heterophilic graphs. To improve the time complexity of our approach, we introduce a variant in which residual strengths are not learned but instead set heuristically, a choice that performs as well as the learnable version. 
\end{abstract}

\section{Introduction}

Neural message passing~\cite{gilmer2017neural} serves as the foundation of modern graph representation learning, providing a core mechanism for aggregating neighborhood information in graph-structured data. This principle underlies Graph Convolutional Networks (GCNs)~\cite{kipf2016semi} and their variants, such as GraphSAGE~\cite{hamilton2017inductive_table} and Graph Attention Networks (GAT)~\cite{velivckovic2017graph}, which learn node embeddings through iterative feature aggregation and transformation. 

The versatility of graph neural networks (GNNs) has led to their successful deployment across numerous domains, for example, accurate traffic prediction in transportation systems~\cite{van2024towards}, pandemic analysis in public health~\cite{panagopoulos2021transfer,sehanobish2021gaining}, economic forecasting~\cite{ye2021coupled}, and social network analysis~\cite{hevapathige2025deepsn,hevapathige2025graph,dehkordi2025graph}.

While powerful, these architectures face inherent challenges, such as oversmoothing, a phenomenon in which network propagation results in indistinguishable node embeddings~\cite{oono2019graph}. This phenomenon may even occur in shallow networks~\cite{wu2022non}. A common strategy to mitigate oversmoothing in GNNs is the use of residual connections~\cite{gasteiger2018predict,huang2019residual,chen2020simple}, a technique that adds the input from the previous layers to the aggregated neighbor information (e.g., through summation or concatenation) to preserve information across layers. This approach has been shown to improve training and performance in deep learning architectures such as ResNets~\cite{he2016deep}.

Residual connections may combine the current embeddings with the initial node embeddings, forming initial residual connections (IRC). More generally, they may combine the current embeddings with the outputs from previous layers, which we refer to as residual connections (RC). In this work, we focus on IRC. Earlier works on IRC typically treat the residual strengths as a fixed hyperparameter shared across all nodes, referred to as \textit{static} IRC for simplicity (see, e.g.,~\cite{chen2020simple,scholkemper2024residual}).
To further unlock the power of IRC, some research has considered a more dynamic and adaptive approach, cf.~\cite{zhang2023drgcn}. Following this line of work, in this paper, we study the \textit{adaptive IRC}, in which each node is permitted to have a personalized residual strength.
Our primary goal is to theoretically and experimentally investigate the ability of adaptive IRC to preserve feature information across layers and achieve superior accuracy in downstream tasks.

To assess whether a message passing mechanism retains the ability to distinguish node embeddings as the network depth increases, a popular choice is the rank of the node embedding matrix~\cite{daneshmand2020batch}. A higher rank indicates that the message passing mechanism preserves distinct information across nodes, enabling richer representations. However, we note that rank alone is insufficient to diagnose oversmoothing, even in oversmoothing regimes, because numerical rank may remain full due to infinitesimal differences (though the effective rank collapses). 
For a more reliable diagnosis, effective rank or energy decay of the embedding matrix must be used to complement standard rank analysis. Hence, we also use the Dirichlet energy (see mean average distance~\cite{chen2020measuring}, spectral rank~\cite{zhang2025rethinking}, and normalized node similarities~\cite{chen2025residual} for alternative equivalent measures), which quantifies the smoothness of node embeddings. If it decays to zero with depth, the embeddings become indistinguishable, indicating a loss of expressiveness. Conversely, bounded energy implies that the mechanism maintains discriminative power. Unlike rank, energy is a continuous quantifier~\cite{horn2012matrix}, making it robust for tracking the degree of smoothing. 

Now that our message passing mechanism and our measures of oversmoothing are established, we can summarize our main contributions as follows: 
\begin{itemize}   
        \item \textbf{Theoretical Guarantee}: We provide a theoretical analysis demonstrating that the adaptive IRC framework avoids oversmoothing, even in deep architectures. Specifically, we show that the Dirichlet energy of the node embeddings in the adaptive IRC remains bounded away from zero, ensuring that the embeddings retain their discriminative power. As a special case, we prove that static IRC with activation functions mitigates oversmoothing, extending the result of the prior work~\cite{scholkemper2024residual}, which considered only the linear case without activation functions. We also prove that, through the aggregation process, in the adaptive IRC mechanism, the rank of the embedding matrix is fully preserved.
         \item \textbf{Enhanced GNN Performance}: Through extensive experiments on benchmark datasets, we demonstrate that our framework consistently outperforms state-of-the-art GNNs. Our results highlight the framework's ability to learn robust and discriminative node embeddings while maintaining scalability, as the node-specific residual strength can be efficiently learned during training.
         
        \item \textbf{Adaptive IRC with Non-learnable Residual Strengths}: To reduce the time complexity of our approach, we propose a variant of adaptive IRC where residual strengths are not learned, but instead assigned heuristically based on node centrality. In particular, we use PageRank scores to rank the nodes, assigning a higher fixed residual strength to a small fraction of top-ranked nodes and a lower fixed value to the rest. This simple PageRank-based strategy performs as well as the fully learnable variant while significantly reducing computational overhead.
\end{itemize}

\subsection{Preliminaries}

Let $G = (V, E, \mathbf{A})$ be an undirected weighted graph with a node set $V$ of cardinality $|V| = n$, an edge set $E$, and a weight matrix $\mathbf{A} \in \mathbb{R}^{n \times n}$ where $[\mathbf{A}]_{ij} = a_{ij}$ represents the weight of the edge between nodes $i$ and $j$. If no edge exists between nodes $i$ and $j$, $a_{ij} = 0$. The node set is equal to $V=\{1, 2,\cdots, n\}$, and the one-hop neighborhood of node $i$ is denoted as $N_i = \{ j \in V \mid (i, j) \in E \}$. Each node is associated with a feature vector of size $1 \times d$, and the feature matrix for all nodes is denoted by $\textbf{X} \in \mathbb{R}^{n \times d}$, where $d$ is the dimension of the feature. 

The degree matrix, $\mathbf{D} \in \mathbb{R}^{n \times n}$, is a diagonal matrix where $[\mathbf{D}]_{ii}:=d_i = \sum_{j \in V} a_{ij}$ represents the weighted degree of node $i$, i.e., the sum of weights of edges from node $i$ to its neighbors. 
Considering the Laplacian matrix $\mathbf{L} = \mathbf{D} - \mathbf{A}$, the normalized Laplacian matrix $\bm{\mathcal{L}}$ is defined as
\[
\bm{\mathcal{L}} = \mathbf{D}^{-1/2} \mathbf{L} \mathbf{D}^{-1/2} = \mathbf{D}^{-1/2} (\mathbf{D} - \mathbf{A}) \mathbf{D}^{-1/2} = \mathbf{I} - \bm{\mathcal{A}},
\]
where $\bm{\mathcal{A}} := \mathbf{D}^{-1/2} \mathbf{A} \mathbf{D}^{-1/2}$ is the normalized adjacency matrix. Normalizing the adjacency matrix ensures that its largest eigenvalue is $1$, leading to the following spectral properties. Let $\gamma_1 \leq \cdots \leq \gamma_n$ be the eigenvalues of $\bm{\mathcal{L}}$, then $0 = \gamma_1 \leq \cdots \leq \gamma_n \leq 2$. Let $\alpha_1 \geq \cdots \geq \alpha_n$ be the eigenvalues of $\bm{\mathcal{A}}$, then $1 = \alpha_1 \geq \cdots \geq \alpha_n \geq -1$; see~\cite{chung1997spectral} for more details. 

Let $\textbf{M} \in \mathbb{R}^{n \times m}$ be a matrix with $n, m \in \mathbb{N}^+$ and $\text{rank}(\textbf{M})=r$. Then, the singular value decomposition of $\textbf{M}$ is given by $\textbf{M} = \textbf{U} \bm{\Sigma} \textbf{V}^{\top}$, where the columns of $\textbf{U} \in \mathbb{R}^{n \times n}$ and $\textbf{V} \in \mathbb{R}^{m \times m}$ are orthogonal (i.e., $\textbf{U}^{\top} \textbf{U} =\textbf{I} \in \mathbb{R}^{n \times n}$, $\textbf{V}^{\top} \textbf{V} =\textbf{I} \in \mathbb{R}^{m \times m}$) and $\bm\Sigma \in \mathbb{R}^{n \times m}$ is zero everywhere except for entries on the main diagonal where the $(j,j)$ entry is $\sigma_j$ for $j=1,\cdots,\min\{m,n\}$ and 
\[
\sigma_1 \geq \sigma_2 \geq \cdots \geq \sigma_r > \sigma_{r+1}=\cdots=\sigma_{\min\{ m,n \}}=0.
\]
Denoting the columns of $\textbf{U}$ and $\textbf{V}$ as $\textbf{u}_1,\cdots,\textbf{u}_n$ and $\textbf{v}_1,\cdots,\textbf{v}_m$, we can write $\textbf{M}=\sum_{j=1}^r \sigma_j \textbf{u}_j \textbf{v}_j^{\top}$. The decomposition satisfies $\textbf{M} \mathbf{v}_i = \sigma_i \mathbf{u}_i$ and $\textbf{M}^{\top} \mathbf{u}_i = \sigma_i \mathbf{v}_i$ for $i = 1, \dots, r$. The Frobenius norm of $\textbf{M}$ is $\|\textbf{M}\|_F^2 = \sum_{i=1}^r \sigma_i^2$, and the spectral norm is $\|\textbf{M}\|_2^2 = \sigma_1$ (see ~\cite{horn2012matrix,davydov2021meshless} for more details).

\textit{The column space} of a matrix $\mathbf{M} \in \mathbb{R}^{n \times m}$, denoted by $\mathcal{R}(\mathbf{M})$, represents all possible linear combinations of its columns and is formally defined as $\mathcal{R}(\mathbf{M}) = \{ \mathbf{M}\mathbf{x}: \mathbf{x} \in \mathbb{R}^m \} \subseteq \mathbb{R}^n$.
Similarly, \textit{the row space} $\mathcal{R}(\mathbf{M}^{\top})$ captures all linear combinations of the rows of $\mathbf{M}$ and can be expressed as $\mathcal{R}(\mathbf{M}^{\top}) = \{ \mathbf{M}^{\top}\mathbf{y} : \mathbf{y} \in \mathbb{R}^n \} \subseteq \mathbb{R}^m$. There is a connection between these spaces and singular vectors. Specifically, if $\mathbf{M}$ has rank $r$, its column and row spaces are spanned by the singular vectors as 
\[
\mathcal{R}(\mathbf{M}) = \operatorname{span}\{\mathbf{u}_1, \dots, \mathbf{u}_r\},\  \mathcal{R}(\mathbf{M}^{\top}) = \operatorname{span}\{\mathbf{v}_1, \dots, \mathbf{v}_r\},
\]
where $\{\mathbf{u}_i\}$ and $\{\mathbf{v}_i\}$ are the left and right singular vectors, respectively. This result is known as the fundamental theorem of linear algebra, the singular value decomposition version~\cite{trefethen2022numerical}. Furthermore, \textit{the null space} of matrix $\textbf{M}$ is denoted by $\mathcal{N}(\textbf{M})=\{ \textbf{x }\in \mathbb{R}^m: \textbf{M} \textbf{x}=\textbf{0}\}$.

\subsection{Adaptive Initial Residual Connection}
We consider the following message passing, called \textbf{adaptive IRC}: \begin{equation}\label{GRC}
\bm{H}^{(\ell+1)} = \sigma \bigg(\bm{\Lambda} \bm{\mathcal{A}} \bm{H^{(\ell)}} {\bm{W}^{(\ell)}} + (\textbf{I}-\bm{\Lambda}) \bm{H}^{(0)} \bm{\Theta}^{(\ell)}\bigg)
\end{equation}
where $\bm{H}^{(\ell)} \in \mathbb{R}^{n \times d_l}$ represents the node embedding matrix at layer $\ell$, with $n$ being the number of nodes and $d_l$ being the dimensionality of the node embeddings at this layer. Specifically, $\bm{H}^{(0)}$ corresponds to the initial node feature matrix, which is the input to the network. Also, $ \sigma(\cdot)$\footnote{
We use $\sigma$ to denote the activation function and $\sigma_i$ to denote singular values indexed by $i$.
} denotes a nonlinear activation function, such as ReLU or sigmoid, applied element-wise to its input. $\bm{\Lambda}=\text{diag}(\lambda_1,\lambda_2,\cdots,\lambda_n)$ is the \textit{residual strength diagonal matrix} where entries $\lambda_i \in (0, 1)$ determine the weight assigned to the neighborhood aggregation for each node vs its own initial embedding. Through the paper, we assume $\lambda_{\min}$ and $\lambda_{\max}$ are, respectively, the minimum and maximum values of $\lambda_i$, $i=1,\cdots,n$. $\bm{W}^{(\ell)}, \bm{\Theta}^{(\ell)} \in \mathbb{R}^{d_l \times d_{\ell+1}}$ are learnable weight matrices at layer $l$, which transform the aggregated neighborhood embeddings and the initial node features, respectively and \( \textbf{I} \) is the identity matrix. 

It is worth noting that setting $\bm{\Lambda} = \mathbf{I}$ will recover the vanilla GCN~\cite{kipf2016semi}, where no residual term is present. Moreover, when $\bm{\Lambda} = \beta \mathbf{I}$ for some $\beta \in (0,1)$, meaning all nodes share the same residual connection strength, the static IRC model~\cite{gasteiger2018predict,scholkemper2024residual} is recovered. \textit{It's worth emphasizing that the intuition behind design~\eqref{GRC} is that, instead of relying solely on aggregating information from neighboring nodes, each node adaptively retains a learnable portion of its initial embedding, an idea inspired by the Friedkin-Johnsen opinion dynamics model~\cite{friedkin1990social,shirzadi2025opinion,shirzadi2025stubborn}.}

To ensure generalization to unseen nodes, we define $\bm{\Lambda}$ based on the initial node features rather than learning a separate value for each node. Specifically, we set $\bm{\Lambda} = \text{diag}\left(\sigma\left( \bm{H}^{(0)} \bm{W}_{\text{att}} \right) \right)$, where $\bm{W}_{\text{att}} \in \mathbb{R}^{d_0 \times 1}$. This applies a fully connected layer followed by a sigmoid function to each node's initial features, producing residual strengths in the range $0$ to $1$. As the weights are shared and input-dependent, this formulation enables flexible and generalizable $\bm{\Lambda}$ values for unseen nodes.

\subsection{Related Works}

\paragraph{Message Passing in GNNs.} The earliest GNN architectures drew inspiration from spectral graph theory~\cite{defferrard2016convolutional}, utilizing graph Fourier transforms to extract structural patterns in the spectral domain. This feature extraction can be performed either discretely, as seen in GCN~\cite{kipf2016semi}, GraphSAGE~\cite{hamilton2017inductive_table}, and GAT~\cite{velivckovic2017graph}, or continuously through diffusion PDEs, such as in graph neural diffusion~\cite{chamberlain2021grand,thorpe2022grand++} and Allen-Cahn message passing~\cite{wang2022acmp}. Continuous message passing is inspired by the framework of neural differential equations~\cite{chen2018neural}, which has led to many follow-up works in the GNN field~\cite{avelar2019discrete,poli2019graph,wu2023difformer,rusch2022graph,gallicchio2020fast}.


\paragraph{Oversmoothing.} A major challenge for GNNs is their limited depth. As layers increase, models like GCN~\cite{oono2019graph} and GAT~\cite{wang2019improving,wu2023demystifying,dong2021attention} often suffer performance degradation due to repeated neighborhood averaging, which makes node embeddings increasingly similar and eventually indistinguishable. The phenomenon, first noted by~\cite{li2018deeper}, arises from repeated Laplacian smoothing that drives embeddings in a connected graph toward uniform values. Later studies~\cite{oono2019graph,cai2020note} further showed that the embedding energy decays to zero with depth.

\paragraph{Residual Connection.} Motivated by the great success of residual neural networks in conventional deep learning~\cite{he2016deep}, there has been a growing interest in incorporating RC in GNNs. An early example is provided by~\cite {li2019deepgcns}, where the authors demonstrated that RC leads to significant improvements in experiments.~\cite{liu2021graph} provided a message passing with an adaptive embedding aggregation and RC. \cite{yang2022difference,chen2023lsgnn} proposed difference RC, a method that helps GNNs focus only on the remaining useful information (the difference between input and output) at each layer. This prevents the loss of important details when stacking multiple layers.

\paragraph{Initial Residual Connection.} GCNII~\cite{chen2020simple} demonstrated the effectiveness of IRC combined with identity mapping, enabling deeper architectures while preserving model performance. The study by~\cite{scholkemper2024residual} demonstrated that IRC, in GNNs without activation functions, can mitigate oversmoothing when measured by mean average distance. The adaptive IRC with learnable residual strength through different layers has been studied by~\cite{zhang2023drgcn}, which is closely related to our work. However, unlike~\cite{zhang2023drgcn}, our adaptive message passing has significantly lower complexity. Firstly, in our framework, the personalized residual strengths are shared across layers, reducing the number of learnable parameters. Furthermore, we propose a PageRank-based heuristic variant of the framework, which further reduces complexity. More importantly, we provide a theoretical guarantee that adaptive IRC mitigates oversmoothing, in terms of Dirichlet energy, which was observed only experimentally by~\cite{zhang2023drgcn} (for a more complex variant). As a byproduct of these results, we extend the theoretical results of~\cite{scholkemper2024residual} on static IRC.

Similar methods that aggregate not just initial node features but also a combination of other layer embeddings show satisfactory performance, as seen in Jumping Knowledge Networks (JKNets)~\cite{xu2018representation}, Deep Adaptive Graph Neural Networks (DAGNNs)~\cite{liu2020towards}, R--SoftGraphAI~\cite{li2024curriculum} and GODNF~\cite{hevapathige2025graph}.

\section{Theoretical Foundations}
In this section, we present two key theoretical results concerning the expressive power of the message passing framework defined in Equation~\eqref{GRC}. First, for the simplified case without activation functions or linear transformations, Theorem~\ref{rank_preservation_thorem} establishes that the embedding matrix maintains its initial rank throughout propagation. This result demonstrates that the adaptive IRC framework effectively prevents dimensional collapse in the embedding space. Furthermore, Theorem~\ref{theorem_dirichlet_energy} reveals that the energy function of the message passing~\eqref{GRC} (without any simplification, i.e., in the presence of nonlinearity and weights for different layers) remains bounded away from zero as the network depth increases.

\subsection{Rank Preservation in Simplified Adaptive IRC}

We begin our theoretical analysis by considering a simplified setup, where we remove the intermediate activation functions and linear transformations to isolate the core averaging behavior of message passing. This relaxation leads to the following propagation rule
\begin{equation} \label{fj_normalized_simplified}
\mathbf{H}^{(\ell+1)} = \bm\Lambda \bm{\mathcal{A}} \mathbf{H}^{(\ell)}  + (\mathbf{I} - \bm\Lambda) \mathbf{H}^{(0)},
\end{equation}
This simplification is consistent with prior work~\cite{wu2019simplifying}, which argues that most of the benefit in GCNs comes from local averaging rather than from nonlinear activation functions.

\begin{theorem}\label{rank_preservation_thorem}
Considering the simplified version of the message passing~\eqref{GRC} given by~\eqref{fj_normalized_simplified}, where we have no activation function, no linear transformation, the system stabilizes and it maintains full rank embeddings for all $\ell \in \mathbb{N}$. More precisely, the limiting behavior is governed by
\[
\lim_{\ell \to \infty} \mathbf{H}^{(\ell)} = (\mathbf{I} - \bm\Lambda \bm{\mathcal{A}})^{-1}(\mathbf{I} - \bm \Lambda)\mathbf{H}^{(0)},
\]
and the embedding space never collapses, as
\[
\mathrm{rank}\big(\mathbf{H}^{(\ell)}\big) = \mathrm{rank}\big(\mathbf{H}^{(0)}\big).
\]
\end{theorem}
\begin{proof}
Unfolding the update rule~\eqref{fj_normalized_simplified}, gives 
\begin{equation}\label{unroll}
\mathbf{H}^{(\ell+1)} = (\bm\Lambda \bm{\mathcal{A}})^{\ell+1}\mathbf{H}^{(0)} + \left(\sum_{i=0}^\ell \big(\bm\Lambda \bm{\mathcal{A}}\big)^i\right)(\mathbf{I} - \bm\Lambda)\mathbf{H}^{(0)}
\end{equation}
We claim that the finite sum of matrix powers in Equation~\eqref{unroll} admits an exact closed-form expression. First, as $\bm{\mathcal{A}}$ is symmetric, we have 
\begin{align}\label{norm_A_2}
\|\bm{\mathcal{A}}\|_2 
&= \max_{\textbf{x} \neq 0} \frac{\|\bm{\mathcal{A}}\textbf{x}\|_2}{\|\textbf{x}\|_2} 
= \max_{\textbf{x} \neq 0} \frac{\left(\textbf{x}^{\top} \bm{\mathcal{A}}^{\top} \bm{\mathcal{A}} \textbf{x}\right)^{1/2}}{\left(\textbf{x}^{\top} \textbf{x}\right)^{1/2}} \nonumber \\
&= \sqrt{\alpha_{\max}(\bm{\mathcal{A}}^{\top} \bm{\mathcal{A}})} 
= \alpha_{\max}(\bm{\mathcal{A}}) = 1
\end{align}
For any matrix $\mathbf{M}$, the spectral radius $\rho(\mathbf{M})$ satisfies $\rho(\mathbf{M}) \leq \|\mathbf{M}\|_2$ (this follows because for any eigenvalue $\lambda$ of $\mathbf{M}$ with corresponding eigenvector $\mathbf{v}$, we have $|\lambda| \|\mathbf{v}\|_2=\|\mathbf{M}\mathbf{v}\|_2 \leq \|\mathbf{M}\|_2 \|\mathbf{v}\|_2$, so $|\lambda| \leq \|\mathbf{M}\|_2$). Therefore, we have
\[
\rho(\bm\Lambda \bm{\mathcal{A}}) \leq \|\bm\Lambda \bm{\mathcal{A}}\|_2 \leq\|\bm\Lambda\|_2\|\bm{\mathcal{A}}\|_2 = \|\bm\Lambda\|_2 \times 1 < 1.
\]
This spectral radius condition guarantees the convergence of the Neumann series~\cite{horn2012matrix}
\[
\sum_{i=0}^{\infty} (\bm\Lambda \bm{\mathcal{A}})^i = (\mathbf{I} - \bm\Lambda \bm{\mathcal{A}})^{-1}.
\]
To derive the finite sum expression, we proceed as follows
\begin{align*}
\sum_{i=0}^\ell (\bm\Lambda \bm{\mathcal{A}})^i 
&= ( \mathbf{I} - \bm\Lambda \bm{\mathcal{A}} )^{-1} - \sum_{i=\ell+1}^\infty (\bm\Lambda \bm{\mathcal{A}})^i \\
&= ( \mathbf{I} - \bm\Lambda \bm{\mathcal{A}} )^{-1} - (\bm\Lambda \bm{\mathcal{A}})^{\ell+1} \sum_{j=0}^\infty (\bm\Lambda \bm{\mathcal{A}})^j \\
&= ( \mathbf{I} - \bm\Lambda \bm{\mathcal{A}} )^{-1} - (\bm\Lambda \bm{\mathcal{A}})^{\ell+1} ( \mathbf{I} - \bm\Lambda \bm{\mathcal{A}} )^{-1} \\
&= \Big(\mathbf{I} - (\bm\Lambda \bm{\mathcal{A}})^{\ell+1}\Big) (\mathbf{I} - \bm\Lambda \bm{\mathcal{A}})^{-1}
\end{align*}
Substituting this result into Equation~\eqref{unroll} yields
\begin{align*}
\mathbf{H}^{(\ell+1)} &= (\bm\Lambda \bm{\mathcal{A}})^{\ell+1} \mathbf{H}^{(0)} \notag \\
&\quad + \big(\mathbf{I} - (\bm\Lambda \bm{\mathcal{A}})^{\ell+1} \big)(\mathbf{I} - \bm\Lambda \bm{\mathcal{A}})^{-1}(\mathbf{I} - \bm\Lambda)\mathbf{H}^{(0)}.
\end{align*}
Taking the limit as $\ell \to \infty$, the spectral radius condition (i.e., $\rho(\bm\Lambda\bm{\mathcal{A}}) <1$) ensures $\lim_{\ell \to \infty} (\bm\Lambda \bm{\mathcal{A}})^\ell = \mathbf{0}$, giving the closed-form 
\[
\lim_{\ell \to \infty} \mathbf{H}^{(\ell)} = (\mathbf{I} - \bm\Lambda \bm{\mathcal{A}})^{-1}(\mathbf{I} - \bm \Lambda)\mathbf{H}^{(0)}.
\]
For the rank preservation property, note that as $\lambda_i < 1$, for $i=1,\cdots,n$, so $\mathbf{I} - \bm\Lambda$ is full-rank, and $\rho(\bm\Lambda \bm{\mathcal{A}}) < 1$ ensures $(\mathbf{I} - \bm\Lambda \bm{\mathcal{A}})^{-1}$ is full-rank. Putting all these together, the final results are obtained as 
for any matrices $\mathbf{A} \in \mathbb{R}^{n \times n}$ (invertible) and 
$\mathbf{B} \in \mathbb{R}^{n \times n}$, we have $\mathrm{rank}(\mathbf{AB}) = \mathrm{rank}(\mathbf{B})$. 
\end{proof}

\subsection{Dirichlet Energy}
Our analysis in this section proceeds as follows: first, we establish Lemma~\ref{fW_upper_bound}, Lemma~\ref{LambdaA_upper_bound}, and Corollary~\ref{upperbound_LambdaAXW}, which then enable the proof of our main result in Theorem~\ref{theorem_dirichlet_energy}.

\begin{definition}[Dirichlet Energy]\label{dirichlet_energy_def}
The Dirichlet energy of a scalar vector $\textbf{f} \in \mathbb{R}^{n \times 1}$ on the graph $G = (V, E, \mathbf{A})$ is given by  
\[\small
\mathcal{E}(\textbf{f}) = \textbf{f}^{\top} \bm{\mathcal{L}} \textbf{f} = \frac{1}{2} \sum_{(i,j) \in E} a_{ij} \left( \frac{f_i}{\sqrt{1 + d_i}} - \frac{f_j}{\sqrt{1 + d_j}} \right)^2.
\]      
\end{definition}
For a feature matrix \( \textbf{X} \in \mathbb{R}^{n \times d} \) with rows \( \textbf{x}_i \in \mathbb{R}^{1 \times d} \), the Dirichlet energy extends naturally as 
\[ \small
\mathcal{E}(\textbf{X}) = \text{tr}(\textbf{X}^{\top} \bm{\mathcal{L}} \textbf{X}) = \frac{1}{2} \sum_{(i,j) \in E} a_{ij} \left\| \frac{\textbf{x}_i}{\sqrt{1 + d_i}} - \frac{\textbf{x}_j}{\sqrt{1 + d_j}} \right\|_2^2.
\] 

This formulation quantifies the smoothness of the features over the graph; higher energy values indicate greater differences in features between adjacent nodes~\cite{rusch2023survey}.

\begin{lemma}\label{fW_upper_bound}
Let \(\textbf{W} \in \mathbb{R}^{d_{\ell} \times d_{\ell+1}}\) be a weight matrix of rank $r$. For any vector \(\textbf{f} \in \mathbb{R}^{1 \times d_{l}}\) in the column space of \(\textbf{W}\), we have
\[
\| \textbf{f} \textbf{W} \|_2 \geq \| \textbf{f}\|_2 \sigma_r(\textbf{W}),
\]
where \(\sigma_r(\textbf{W})\) is the smallest non-zero singular value of \(\textbf{W}\).
\end{lemma}

\begin{proof}
By the singular value decomposition~\cite{horn2012matrix}, $\textbf{W}$ can be decomposed as $\textbf{W} = \textbf{U} \bm\Sigma \textbf{V}^{\top}$, where $\textbf{U} \in \mathbb{R}^{d_{l} \times d_{l}}$ and 
$\textbf{V} \in \mathbb{R}^{d_{l+1} \times d_{l+1}}$ are both orthogonal matrices. The matrix $\bm\Sigma \in \mathbb{R}^{d_{l} \times d_{l+1}}$ is diagonal, with singular values $\sigma_1 \geq \sigma_2 \geq \dots \geq \sigma_{r} > \sigma_{r+1}=0\cdots =\sigma_{\min\{d_{\ell},d_{\ell+1}\}}$, where $r = \text{rank}(\textbf{W})$; the remaining entries of $\bm\Sigma$ are zero. For any vector $\textbf{f} \in \mathbb{R}^{1 \times d_l}$ in the column space of $\textbf{W}$, we can write
\[
\textbf{f}\textbf{W} = \textbf{f} (\textbf{U} \bm\Sigma \textbf{V}^{\top}) = (\textbf{f}\textbf{U}) \bm\Sigma \textbf{V}^{\top}.
\]
Let $\textbf{y} := \textbf{f}\textbf{U}$. As $\textbf{f}$ is in the column space of $\textbf{W}$, it can be written as $\textbf{f}=\sum_{j=1}^r \alpha_j \textbf{u}_j^{\top}$ where $\alpha_j$'s are some real coefficients. As the columns of $\textbf{U}$ are orthogonal, $\textbf{y}=(\alpha_1,\cdots,\alpha_r,0,\cdots,0)$ and note that $\| \textbf{y} \|_2=\left(\sum_{i=1}^r \alpha_i^2\right)^{1/2}$. 

Hence, since $\textbf{f}\textbf{W} = \textbf{y} \bm\Sigma \textbf{V}^{\top}$, we have $\|\textbf{f}\textbf{W}\|_2 = \|\textbf{y} \bm\Sigma \textbf{V}^{\top}\|_2=\|\textbf{y} \bm\Sigma\|_2$ (note as $\textbf{V}^{\top}$ is orthogonal, it preserves the norm). Since $\textbf{y} \bm\Sigma = [\alpha_1 \sigma_1, \alpha_2 \sigma_2, \dots, \alpha_{r} \sigma_{r}, 0, \dots, 0]$, we have 
\begin{align*}\label{lower_bound_fW_in_proof}
\|\textbf{f}\textbf{W}\|_2 = \|\textbf{y} \bm\Sigma\|_2 
&= \sqrt{\sum_{i=1}^{r} (\alpha_i \sigma_i)^2} 
\geq \sigma_{r} (\textbf{W}) \| \textbf{y}\|_2.
\end{align*}
Since orthogonal transformations preserve the Euclidean norm $\|\textbf{y}\|_2 = \|\textbf{f}\|_2$, which finishes the proof.
\end{proof}

\begin{remark}
The assumption that $\textbf{f}$ lies in the column space of $\textbf{W}$ is necessary for the inequality to hold, as otherwise, the inequality does not hold. For example, consider $\mathbf{W} = \begin{bmatrix} 1 & 0 \\ 0 & 0 \end{bmatrix}$ (with rank $r=1$ and $\sigma_r=1$) and the vector $\mathbf{f} = [0, 1]$. Here, $\mathbf{f}$ lies outside the column space of $\mathbf{W}$ (which is spanned solely by $[1, 0]$), resulting in $\mathbf{f}\mathbf{W} = [0, 0]$. Consequently, we obtain $\|\mathbf{f}\mathbf{W}\|_2 = 0$ while $\sigma_r \|\mathbf{f}\|_2 = 1$, violating the inequality. 
\end{remark}

\begin{lemma}\label{LambdaA_upper_bound}
Let $\bm\Lambda$ be a diagonal residual strength matrix, and let $\bm{\mathcal{A}}$ be a normalized adjacency matrix of rank $r$. Then, for any vector $\textbf{f} \in \mathbb{R}^{n \times 1}$ in the row space of $\bm{\mathcal{A}}$, we have 
\[
\big\| \bm\Lambda \bm{\mathcal{A}} \textbf{f} \big\|_2 \geq \lambda_{\min} \, \sigma_r(\bm{\mathcal{A}}) \, \|\textbf{f}\|_2, 
\]
where $\lambda_{\min}$ denotes the smallest diagonal entry of $\bm\Lambda$, and $\sigma_r(\bm{\mathcal{A}})$ is the smallest non-zero singular value of $\bm{\mathcal{A}}$.
\end{lemma}

\begin{proof}
Consider the singular value decomposition of $\bm{\mathcal{A}}$ as $\bm{\mathcal{A}} = \textbf{U} \bm{\Sigma} \textbf{V}^{\top}$, where $\textbf{U}, \textbf{V} \in \mathbb{R}^{n \times n}$ are orthogonal matrices, and $\bm{\Sigma}$ is a diagonal matrix with singular values satisfying $\sigma_1 \geq \cdots \geq \sigma_r > \sigma_{r+1} = \cdots = \sigma_n = 0$, where $r = \mathrm{rank}(\bm{\mathcal{A}})$. 

Since $\textbf{f}$ lies in the row space of $\bm{\mathcal{A}}$, it can be expressed as $\textbf{f} = \sum_{j=1}^r \beta_j \textbf{v}_j$, where $\beta_j \in \mathbb{R}$ and $\textbf{v}_j$ are the right singular vectors of $\bm{\mathcal{A}}$. Letting $\textbf{y} := \textbf{V}^{\top} \textbf{f}$ and using the orthogonality of $\textbf{V}$, we obtain $\textbf{y} = (\beta_1, \dots, \beta_r, 0, \dots, 0)^\top$ and $\|\textbf{y}\|_2=(\sum_{j=1}^r \beta_j^2)^{1/2}$. Now observe that
\[
\| \bm\Lambda \bm{\mathcal{A}} \textbf{f} \|_2 
= \| \bm\Lambda \textbf{U} \bm{\Sigma} \textbf{V}^\top \textbf{f} \|_2 
= \| \bm\Lambda \textbf{U} \bm{\Sigma} \textbf{y} \|_2 
\geq \lambda_{\min} \| \textbf{U} \bm{\Sigma} \textbf{y} \|_2,
\]
where the inequality uses the fact that $\bm\Lambda$ is diagonal with minimum entry $\lambda_{\min}$. Since $\textbf{U}$ is orthogonal, it preserves the Euclidean norm, and thus
\[
\| \bm\Lambda \bm{\mathcal{A}} \textbf{f} \|_2 
\geq \lambda_{\min} \| \bm{\Sigma} \textbf{y} \|_2.
\]
Now note that $\bm{\Sigma} \textbf{y} = (\beta_1 \sigma_1, \beta_2 \sigma_2, \dots, \beta_r \sigma_r, 0, \dots, 0)^\top$ and hence
\[
\| \bm{\Sigma} \textbf{y} \|_2 =  \sqrt{\sum_{j=1}^r (\beta_j \sigma_j)^2} \geq  \sigma_r(\bm{\mathcal{A}}) \| \textbf{y} \|_2.
\]
Finally, since orthogonal transformations preserve the Euclidean norm, we have $\| \textbf{y} \|_2 = \| \textbf{f} \|_2$. This completes the proof.
\end{proof}

Combining Lemmata~\ref{fW_upper_bound} and~\ref{LambdaA_upper_bound} gives the following corollary.

\begin{corollary} \label{upperbound_LambdaAXW}
Let $\bm{\Lambda}$ be the (diagonal) residual strength matrix with minimal entry $\lambda_{\min} > 0$, $\bm{\mathcal{A}} \in \mathbb{R}^{n \times n}$ the normalized adjacency matrix, and $\mathbf{W} \in \mathbb{R}^{d_{\ell} \times d_{\ell+1}}$ a weight matrix. For any $\mathbf{X} \in \mathbb{R}^{n \times d_{\ell}}$ whose rows lie in $\operatorname{col}(\mathbf{W})$ and whose columns lie in $\operatorname{row}(\bm{\mathcal{A}})$, the Dirichlet energy satisfies
\[
\mathcal{E}(\bm{\Lambda} \bm{\mathcal{A}} \mathbf{X} \mathbf{W}) \geq \lambda_{\min}^2 \sigma_r^2(\bm{\mathcal{A}}) \sigma_r^2(\mathbf{W}) \mathcal{E}(\mathbf{X}),
\]
where $\sigma_{r}(\cdot)$ denotes the smallest non-zero singular value.
\end{corollary}
\begin{proof}
As $\bm{\Lambda}$ is a diagonal matrix, we have $\| \bm{\Lambda} \textbf{y}\|_2^2 \geq \lambda_{\min}^2 \|\textbf{y}\|_2^2$, for any vector $\textbf{y} \in \mathbb{R}^n$. Hence, 
\[
\mathcal{E}(\bm{\Lambda}\bm{\mathcal{A}}\mathbf{X}\mathbf{W}) \geq \lambda_{\min}^2 \mathcal{E}(\bm{\mathcal{A}}\mathbf{X}\mathbf{W}).
\]
Now, applying Lemma~\ref{fW_upper_bound} and Lemma~\ref{LambdaA_upper_bound} yields
\[
\mathcal{E}(\bm{\mathcal{A}}\mathbf{X}\mathbf{W}) \geq \sigma_r^2(\bm{\mathcal{A}})\sigma_r^2(\mathbf{W}) \mathcal{E}(\mathbf{X}),
\]
which completes the proof. 
\end{proof}

Now, putting all of these together in conjunction with the following properties, we will prove Theorem~\ref{theorem_dirichlet_energy}. 

\begin{property}\label{activation_lower_bound}
The activation function \( \sigma: \mathbb{R} \rightarrow \mathbb{R} \) is such that for any vector \( \bm{f} \in \mathbb{R}^{n \times 1} \), the Dirichlet energy satisfies $\mathcal{E} \big(\sigma(\bm{f})\big) \geq \alpha^2 \mathcal{E}\big(\bm{f}\big)$ for some positive constant \( \alpha > 0 \).
\end{property}

This property ensures the Dirichlet energy is not reduced by more than a constant factor, preserving signal variation after activation. For example, the leaky ReLU function with negative slope $0<\alpha<1$ is defined as 
\[
\text{Leaky ReLU}(x) = 
\begin{cases} 
x & \text{if } x \geq 0, \\
\alpha x & \text{if } x < 0,
\end{cases}
\]
satisfies Property~\ref{activation_lower_bound}, as shown in the following lemma. 

\begin{lemma}\label{upperbound_sigma}
Let \( \sigma \) be leaky ReLU with negative slope $0<\alpha<1$, then,
\[
\mathcal{E} \big(\sigma(\bm{f})\big) \geq \alpha^2 \mathcal{E}\big(\bm{f}\big). 
\]
for any vector $\textbf{f} \in \mathbb{R}^{n \times 1}$.
\end{lemma}

\begin{proof}
We begin by recalling the definition of the energy function:
\begin{equation}
\mathcal{E}(\sigma(\bm{f})) = \frac{1}{2} \sum_{(i,j) \in E} a_{ij} \left( \frac{\sigma(f_i)}{\sqrt{1+d_i}} - \frac{\sigma(f_j)}{\sqrt{1+d_j}} \right)^2.
\end{equation}
To establish the desired inequality, it suffices to prove that for every edge \((i,j) \in \mathcal{E}\),
\begin{equation}\label{claim_activation}
\left( \frac{\sigma(f_i)}{\sqrt{1+d_i}} - \frac{\sigma(f_j)}{\sqrt{1+d_j}} \right)^2 \geq \alpha^2 \left( \frac{f_i}{\sqrt{1+d_i}} - \frac{f_j}{\sqrt{1+d_j}} \right)^2.
\end{equation}
We proceed by examining all possible sign patterns for the normalized inputs \(\frac{f_i}{\sqrt{1+d_i}}\) and \(\frac{f_j}{\sqrt{1+d_j}}\).

\textbf{Case I: Both inputs non-negative.} \\
When \(\frac{f_i}{\sqrt{1+d_i}} \geq 0\) and \(\frac{f_j}{\sqrt{1+d_j}} \geq 0\), the Leaky ReLU reduces to the identity function
\[
\frac{\sigma(f_i)}{\sqrt{1+d_i}} = \frac{f_i}{\sqrt{1+d_i}}, \quad \frac{\sigma(f_j)}{\sqrt{1+d_j}} = \frac{f_j}{\sqrt{1+d_j}}.
\]
Consequently, the inequality~\eqref{claim_activation} holds since \(\alpha^2 \leq 1\).

\textbf{Case II: Both inputs negative.} \\
When \(\frac{f_i}{\sqrt{1+d_i}} < 0\) and \(\frac{f_j}{\sqrt{1+d_j}} < 0\), the Leaky ReLU applies the slope \(\alpha\)
\[
\frac{\sigma(f_i)}{\sqrt{1+d_i}} = \alpha \frac{f_i}{\sqrt{1+d_i}}, \quad \frac{\sigma(f_j)}{\sqrt{1+d_j}} = \alpha \frac{f_j}{\sqrt{1+d_j}}.
\]
In this case, the inequality~\eqref{claim_activation} becomes an equality, as the \(\alpha\) factor appears on both sides.

\textbf{Case III: First input non-negative, second negative.} \\
Suppose \(\frac{f_i}{\sqrt{1+d_i}} \geq 0\) and \(\frac{f_j}{\sqrt{1+d_j}} < 0\), then:
\[
\frac{\sigma(f_i)}{\sqrt{1+d_i}} = \frac{f_i}{\sqrt{1+d_i}}, \quad \frac{\sigma(f_j)}{\sqrt{1+d_j}} = \alpha \frac{f_j}{\sqrt{1+d_j}}.
\]
Let \(a = \frac{f_i}{\sqrt{1+d_i}} \geq 0\) and \(b = \frac{f_j}{\sqrt{1+d_j}} < 0\). The inequality becomes
\[
(a - \alpha b)^2 \geq \alpha^2 (a - b)^2.
\]
Expanding both sides
\[
a^2 - 2\alpha ab + \alpha^2 b^2 \geq \alpha^2 a^2 - 2\alpha^2 ab + \alpha^2 b^2.
\]
or
\[
a^2(1 - \alpha^2) \geq 2\alpha ab(1 - \alpha).
\]
Since \(a \geq 0\), \(b < 0\), and \(\alpha \in (0,1)\), the right-hand side is non-positive while the left-hand side is non-negative, proving the inequality.

\textbf{Case IV: First input negative, second non-negative.} \\
This case is symmetric to Case III and follows by interchanging the roles of \(i\) and \(j\).

Having verified \eqref{claim_activation} for all possible cases, we conclude that \(\mathcal{E}(\sigma(f)) \geq \alpha^2 \mathcal{E}(f)\) as required.
\end{proof}

\begin{property}\label{trace_positive}
For any $\ell$, the we assume $\operatorname{tr}\left(\textbf{X}^{\top} \mathcal{L} \textbf{Y}\right) \geq 0$ where $\textbf{X}=\bm\Lambda \bm{\mathcal{A}}\textbf{H}^{(l)} \textbf{W}^{(\ell)}$ and $\textbf{Y}=(\textbf{I}-\bm\Lambda) \textbf{H}^{(0)} \bm\Theta^{(\ell)}$.
 \end{property}

This property ensures that the difference vectors $\mathbf{x}_i - \mathbf{x}_j$ and $\mathbf{y}_i - \mathbf{y}_j$ are, on average, positively aligned across graph edges $(i,j)$. In other words, the embeddings $\mathbf{X}$ and $\mathbf{Y}$ vary in similar directions among neighboring nodes. A negative trace, by contrast, would suggest that the initial and propagated embeddings vary in opposite directions across edges, potentially contradicting the homophily assumption, which states that connected nodes should have similar embeddings.

\begin{theorem} \label{theorem_dirichlet_energy}
Assume the energy function of the initial embedding is strictly positive, i.e., 
$\mathcal{E}(\bm{H}^{(0)}) > 0$, and the smallest non-zero singular values of the weight matrices are uniformly lower bounded, i.e., for some strictly positive $\epsilon$, 
\begin{align*}
\inf_\ell \left\{ \sigma_r^2(\bm{W}^{(\ell)}) \right\} &= \overline{\sigma}_r (\textbf{W}), \ \ 
\inf_\ell \left\{ \sigma_r^2(\bm{\Theta}^{(\ell)}) \right\} =\overline{\sigma}_r (\bm{\Theta}) > \epsilon.
\end{align*}
where $\sigma_{r}(\cdot)$ denotes the smallest non-zero singular value. Further, suppose $\lambda_{\min} \geq \delta$ and $\lambda_{\max} \leq 1-\delta'$ for some strictly positive $\delta$ and $\delta'$ and the activation function $\sigma(\cdot)$ satisfies Property~\ref{activation_lower_bound} with parameter $\alpha>0$. Then, under the assumption of the Corollary~\ref{upperbound_LambdaAXW} and Property~\ref{trace_positive}, the energy of the final embedding of the message passing~\eqref{GRC} admits the following lower bound
\[
\mathcal{E}(\bm{H}^{(\ell+1)}) \geq \frac{\zeta\overline{\sigma}_r (\bm{\Theta})}{1-\eta\overline{\sigma}_r}\mathcal{E}(\bm{H}^{(0)}) > 0,
\]
where $\eta := \alpha^2\lambda_{\min}^2\sigma_r^2(\bm{\mathcal{A}})$ and $\zeta := \alpha^2(1-\lambda_{\max})^2$.
\end{theorem}

\begin{proof}
Beginning with the message passing~\eqref{GRC} 
and using the Property~\ref{activation_lower_bound}, we derive
\begin{align}
\mathcal{E}(\bm{H}^{(\ell+1)}) 
&= \mathcal{E}\Big(\sigma\big(\bm{\Lambda}\bm{\mathcal{A}}\bm{H}^{(\ell)}\bm{W}^{(\ell)}
+ (\bm{I}-\bm{\Lambda})\bm{H}^{(0)}\bm{\Theta}^{(\ell)}\big)\Big) \nonumber \\
&\geq \alpha^2 \mathcal{E}\Big(\bm{\Lambda}\bm{\mathcal{A}}\bm{H}^{(\ell)}\bm{W}^{(\ell)} 
+ (\bm{I}-\bm{\Lambda})\bm{H}^{(0)}\bm{\Theta}^{(\ell)}\Big).
\label{eq:energy-lower-bound}
\end{align}
Now, by Property~\eqref{trace_positive}, we have the following for any embedding matrices $\bm{X}$ and $\bm{Y}$
\begin{align*}
\mathcal{E}(\textbf{X}+\textbf{Y}) 
&= \text{tr}\big((\textbf{X}+\textbf{Y})^{\top} \bm{\mathcal{L}} (\textbf{X}+\textbf{Y})\big) \nonumber \\
&= \text{tr}(\textbf{X}^{\top} \bm{\mathcal{L}} \textbf{X}) + 2\text{tr}(\textbf{X}^{\top} \bm{\mathcal{L}} \textbf{Y}) + \text{tr}(\textbf{Y}^{\top} \bm{\mathcal{L}} \textbf{Y}) \nonumber \\
&\geq \text{tr}(\textbf{X}^{\top} \bm{\mathcal{L}} \textbf{X}) + \text{tr}(\textbf{Y}^{\top} \bm{\mathcal{L}} \textbf{Y}) \nonumber \\
&= \mathcal{E}(\textbf{X}) + \mathcal{E}(\textbf{Y}). 
\end{align*}
Combining this result with Equation~\eqref{eq:energy-lower-bound} gives us 
\begin{align*}
\mathcal{E}(\bm{H}^{(\ell+1)}) 
&\geq \alpha^2 \mathcal{E}\left(\bm{\Lambda}\bm{\mathcal{A}}\bm{H}^{(\ell)}\bm{W}^{(\ell)}\right) \nonumber \\
&\quad+ \alpha^2 \mathcal{E}\left((\bm{I}-\bm{\Lambda})\bm{H}^{(0)}\bm{\Theta}^{(\ell)}\right).
\end{align*}
Exploiting the spectral properties of the matrices involved in Corollary~\ref{upperbound_LambdaAXW}, this further simplifies to
\[
\mathcal{E}(\bm{H}^{(\ell+1)}) \geq \eta\sigma_r^2(\bm{W}^{(\ell)})\mathcal{E}(\bm{H}^{(\ell)}) + \zeta\sigma_r^2(\bm{\Theta}^{(\ell)})\mathcal{E}(\bm{H}^{(0)}).
\]
Iterating this inequality backward through the layers yields
\begin{align*}
\mathcal{E}(\bm{H}^{(\ell+1)}) 
&\geq \eta^{\ell+1}\left(\prod_{i=0}^{\ell}\sigma_r^2(\bm{W}^{(i)})\right)\mathcal{E}(\bm{H}^{(0)})
\end{align*}
\begin{align*}
\ \ \ \ +\zeta \sum_{k=0}^{\ell}\eta^{k}\sigma_r^2(\bm{\Theta}^{(\ell-k)}) \Bigg(\prod_{j=\ell-k+1}^{\ell}\sigma_r^2(\bm{W}^{(j)})\Bigg)\mathcal{E}(\bm{H}^{(0)}).
\end{align*}
As $\sigma_r^2(\bm{W}^{(i)}) \geq \overline{\sigma}_r(\textbf{W})$ and $\sigma_r^2(\bm{\Theta}^{(i)}) \geq \overline{\sigma}_r(\bm\Theta)$, this last inequality becomes
\[\small
\mathcal{E}(\bm{H}^{(\ell+1)}) \geq \left((\eta\overline{\sigma}_r(\textbf{W}))^{\ell+1} + \zeta \overline{\sigma}_r(\bm\Theta) \sum_{k=0}^{\ell}(\eta\overline{\sigma}_r(\textbf{W}))^{k}\right)\mathcal{E}(\bm{H}^{(0)}).
\]
For $\eta\overline{\sigma}_r(\textbf{W}) < 1$ (since otherwise the lower bound diverges as \(\ell \to \infty\), making the result trivial), the geometric series converges, and the first term vanishes as $\ell \to \infty$, giving
\begin{align*}
\mathcal{E}(\bm{H}^{(\ell+1)}) 
&\geq \zeta \overline{\sigma}_r(\bm\Theta) \left(\frac{1 - (\eta\overline{\sigma}_r)^{\ell+1}}{1 - \eta\overline{\sigma}_r}\right)\mathcal{E}(\bm{H}^{(0)}) \nonumber \\
&\xrightarrow{\ell \to \infty} \frac{\zeta \overline{\sigma}_r(\bm\Theta)}{1 - \eta\overline{\sigma}_r}\mathcal{E}(\bm{H}^{(0)}).
\end{align*}
Note that $\zeta$ is strictly positive because $1 - \lambda_{\max} \geq \delta' > 0$. Also, $\overline{\sigma}_r(\bm\Theta)$ is strictly positive by assumption. Additionally, since $\lambda_{\min} \geq \delta$, we have $0 < \eta \overline{\sigma}_r < 1$, which implies $\frac{1}{1 - \eta \overline{\sigma}_r} > 0$, and $\mathcal{E}(\bm{H}^{(0)})$ is positive by assumption.
\end{proof}
Before providing our experimental results, we analyze the complexity of adaptive IRC below. 

\subsection{Time Complexity}
The time complexity of one layer in the adaptive IRC message passing framework is derived from three key computational components: (1) the sparse neighborhood aggregation $\bm{\mathcal{A}}\bm{H}^{(\ell)}$ requiring $O(|E|d_\ell)$ operations; (2) the subsequent dense embedding transformation $\bm{H}^{(\ell)}\bm{W}^{(\ell)}$ with $O(nd_\ell d_{\ell+1})$ complexity; and (3) the residual projection $\bm{H}^{(0)}\bm{\Theta}^{(\ell)}$ costing $O(nd_0 d_{\ell+1})$. The diagonal scaling operations contribute $O(nd_{\ell})$ term, which is asymptotically negligible. Combining these dominant terms yields an overall layer complexity of $O(|E|d_\ell + nd_\ell d_{\ell+1} + nd_0 d_{\ell+1})$. When all embedding dimensions are unified ($d_\ell = d_{\ell+1} = d_0 = d$), this expression simplifies to the more compact form $O(|E|d + nd^2)$, demonstrating that the adaptive IRC maintains computational efficiency of the vanilla GCN while providing the benefits of residual connections.

\section{Experiments}\label{experiments_section}

Our experiments, whose source code is available at \textcolor{blue}{\texttt{https://rb.gy/dlgx42}}, aim to (i) evaluate whether adaptive IRC mitigates over-smoothing, and (ii) assess its performance on node classification compared to established GNNs.

\subsection{Oversmoothing Mitigation}

We first demonstrate that in adaptive IRC, the nodes' embeddings do not collapse, and the energy level of the nodes' embeddings remains non-zero. We use an undirected synthetic graph generated from a stochastic block model, similar to the setup in~\cite{wang2022acmp,wu2022non}. The graph consists of $200$ nodes divided into two equal classes, with two-dimensional features sampled from a normal distribution. Both classes share the same standard deviation $2$ but have different means ($\mu_1 = -0.5$, $\mu_2 = 0.5$). To model homophily, nodes within the same class are connected with a higher probability ($p = 0.2$), while nodes from different classes are connected with a lower probability ($q = 0.05$). As an early observation, Figure~\ref{embedding_evolution_2} shows that GCN embeddings collapse after a few layers due to rank collapsing, while adaptive IRC preserves class separation even after $16$ layers. A similar collapse behavior is observed for GAT and GraphSAGE with mean and max pooling, as shown in Figure~\ref{embedding_evolution_4}. 

\begin{figure}[htbp]
\includegraphics[width=1\columnwidth]{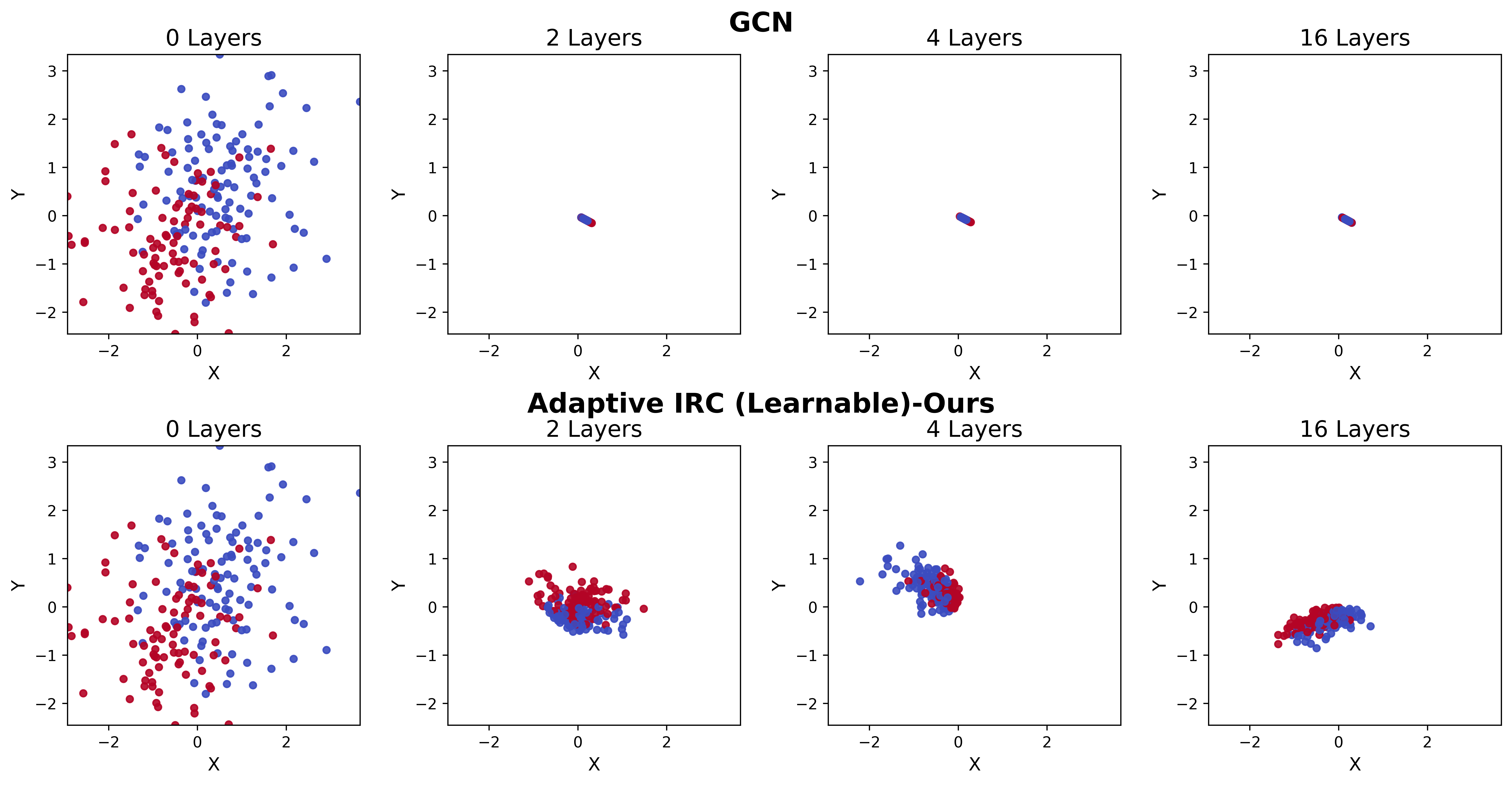}
\caption{Embedding evolution in GCN vs adaptive IRC. GCN leads to embedding collapse, while adaptive IRC preserves distinct clusters.} \label{embedding_evolution_2}
\end{figure}

\begin{figure}[htbp]
\includegraphics[width=1\columnwidth]{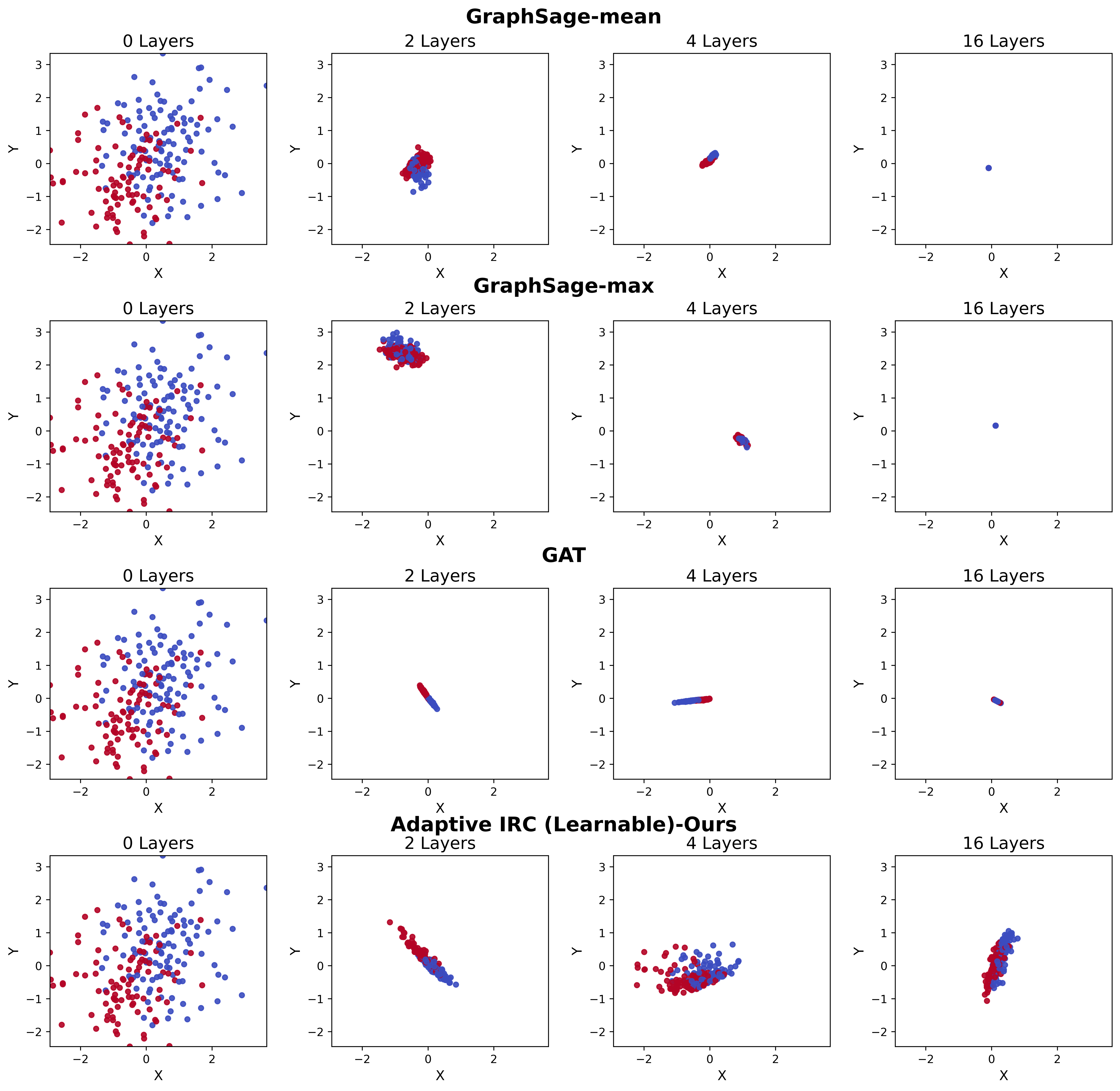}
\caption{Feature evolution in GAT, GraphSAGE (mean and max), and adaptive IRC. GAT and GraphSAGE suffer from Embedding collapse due to oversmoothing, while adaptive IRC preserves distinct clusters.}\label{embedding_evolution_4}
\end{figure}

In Figure~\ref{dirichlet_energy}, we plot the Dirichlet energy for the output of GCN, GAT, GraphSAGE, and adaptive IRC on a logarithmic scale with varying numbers of layers. The figure shows that the energy functions of all GCN, GAT, and GraphSAGE diminish as the number of hidden layers increases. In contrast, the energy level of the adaptive IRC remains notably positive, confirming that it effectively mitigates oversmoothing.

\begin{table*}[ht]\scriptsize
\centering
\caption{Test accuracy and standard deviation over $10$ experiments on each dataset, using different train/validation/test splits. \textcolor{red}{\textbf{Red}} is the best, \textcolor{blue}{\textbf{Blue}}, the second best. The variable $H$ stands for the homophily rate.}
\label{results_Heterophilic}
\begin{tabular}{|l|c|c|c|c|c|c|c|c|c|}
\hline
\rowcolor[HTML]{CFE2F3}
\textbf{Method} & 
\shortstack{\textbf{Cora} \\ \tiny($H$: 0.83)} & 
\shortstack{\textbf{CiteSeer} \\ \tiny($H$: 0.71)} & 
\shortstack{\textbf{Pubmed} \\ \tiny($H$: 0.79)} & 
\shortstack{\textbf{Texas} \\ \tiny($H$: 0.11)} & 
\shortstack{\textbf{Wisconsin} \\ \tiny($H$: 0.21)} & 
\shortstack{\textbf{Cornell} \\ \tiny($H$: 0.30)} & 
\shortstack{\textbf{Chameleon} \\ \tiny($H$: 0.23)} & 
\shortstack{\textbf{Squirrel} \\ \tiny($H$: 0.22)} & 
\shortstack{\textbf{Actor} \\ \tiny($H$: 0.24)} \\ \hline
GCN~\cite{kipf2016semi} & 79.2{\tiny$\pm$0.4} & 64.9{\tiny$\pm$0.8} & 76.7{\tiny$\pm$0.8} & 55.9{\tiny$\pm$6.4} & 47.1{\tiny$\pm$8.5} & 40.3{\tiny$\pm$7.1}  & 33.4{\tiny$\pm$2.2}  & 27.2{\tiny$\pm$0.7}  & 27.3{\tiny$\pm$1.1} \\ \hline
GraphSage-max~\cite{hamilton2017inductive_table} & 75.6{\tiny$\pm$1.0} & 62.3{\tiny$\pm$1.1} & 75.8{\tiny$\pm$0.8} & 72.7{\tiny$\pm$6.6} & 74.3{\tiny$\pm$5.8} & 69.7{\tiny$\pm$3.1} & 50.6{\tiny$\pm$2.2} & 36.8{\tiny$\pm$1.3} & 32.9{\tiny$\pm$1.4} \\ \hline
GraphSage-mean~\cite{hamilton2017inductive_table} & 79.8{\tiny$\pm$0.5} & 66.9{\tiny$\pm$1.0} & 75.8{\tiny$\pm$0.6} & 74.6{\tiny$\pm$6.0} & 76.7{\tiny$\pm$5.3} & 69.5{\tiny$\pm$3.8} & 50.2{\tiny$\pm$1.7} & 36.7{\tiny$\pm$2.3} & 34.1{\tiny$\pm$1.0} \\ \hline
GAT~\cite{velivckovic2017graph} & 75.6{\tiny$\pm$0.4} & 66.2{\tiny$\pm$1.1} & 75.8{\tiny$\pm$1.7} & 57.3{\tiny$\pm$5.9} & 50.4{\tiny$\pm$7.8} & 44.6{\tiny$\pm$5.7} & 39.6{\tiny$\pm$1.8} & 30.7{\tiny$\pm$1.6} & 28.1{\tiny$\pm$1.2} \\ \hline
JKNet~\cite{xu2018representation} & 79.5{\tiny$\pm$0.4} & 66.7{\tiny$\pm$0.7} & 76.7{\tiny$\pm$0.4} & 55.1{\tiny$\pm$7.1} & 51.6{\tiny$\pm$4.6} & 46.8{\tiny$\pm$5.4} & 35.5{\tiny$\pm$1.6} & 28.0{\tiny$\pm$1.0} & 28.4{\tiny$\pm$0.9} \\ \hline
Mixhop~\cite{abu2019mixhop} & 75.0{\tiny$\pm$1.0} & 64.9{\tiny$\pm$0.7} & 74.2{\tiny$\pm$0.6} & 71.1{\tiny$\pm$5.6} & 80.4{\tiny$\pm$7.4} & 65.4{\tiny$\pm$6.1} & 48.4{\tiny$\pm$1.7} & 34.5{\tiny$\pm$2.3} & \textcolor{blue}{\textbf{36.2{\tiny$\pm$1.0}}} \\ \hline
SGC~\cite{wu2019simplifying} & 79.6{\tiny$\pm$0.9} & 65.8{\tiny$\pm$0.8} & \textcolor{blue}{\textbf{77.3{\tiny$\pm$0.2}}} & 57.0{\tiny$\pm$3.5} & 48.8{\tiny$\pm$6.6} & 41.4{\tiny$\pm$6.0} & 35.2{\tiny$\pm$2.6} & 27.7{\tiny$\pm$1.0} & 27.3{\tiny$\pm$1.0} \\ \hline
GCNII~\cite{chen2020simple} & 79.9{\tiny$\pm$0.5} & 67.7{\tiny$\pm$0.5} & 76.5{\tiny$\pm$1.3} & 59.5{\tiny$\pm$5.3} & 60.4{\tiny$\pm$7.4} & 47.0{\tiny$\pm$7.0} & 36.2{\tiny$\pm$2.7} & 28.8{\tiny$\pm$1.0} & 35.2{\tiny$\pm$1.0} \\ \hline
GraphGPS~\cite{rampavsek2022recipe} & 61.6{\tiny$\pm$4.0}  & 43.6{\tiny$\pm$4.7}  & 64.6{\tiny$\pm$9.0} & 58.1{\tiny$\pm$8.8} & 66.9{\tiny$\pm$6.3} & 52.7{\tiny$\pm$13} & 41.4{\tiny$\pm$2.7} & 31.3{\tiny$\pm$1.8} & 29.5{\tiny$\pm$2.4} \\ \hline
DirGNN~\cite{rossi2024edge} & 77.5{\tiny$\pm$1.2} & 66.0{\tiny$\pm$1.7} & 75.0{\tiny$\pm$2.1} & \textcolor{red}{\textbf{84.6{\tiny$\pm$6.1}}} & \textcolor{blue}{\textbf{82.2{\tiny$\pm$2.3}}} & \textcolor{blue}{\textbf{71.6{\tiny$\pm$3.9}}} & 60.6{\tiny$\pm$2.2} & 45.3{\tiny$\pm$1.5} & \textcolor{red}{\textbf{36.6{\tiny$\pm$0.8}}} \\ \hline
\rowcolor[HTML]{E6E6E6}
\scriptsize{Adaptive} IRC (Learnable)-Ours & \textcolor{blue}{\textbf{80.1{\tiny$\pm$1.0}}} & \textcolor{blue}{\textbf{69.3{\tiny$\pm$0.6}}} & 76.6{\tiny$\pm$0.6} & 73.0{\tiny$\pm$5.8} & \textcolor{red}{\textbf{82.4{\tiny$\pm$4.7}}} & 67.8{\tiny$\pm$4.8} & \textcolor{blue}{\textbf{64.1{\tiny$\pm$1.1}}} & \textcolor{blue}{\textbf{47.7{\tiny$\pm$2.2}}} & 34.2{\tiny$\pm$1.3} \\ \hline
\rowcolor[HTML]{E6E6E6}
\scriptsize{Adaptive} IRC (PageRank-based)-Ours & \textcolor{red}{\textbf{80.7{\tiny$\pm$0.4}}} & \textcolor{red}{\textbf{70.2{\tiny$\pm$0.4}}} & \textcolor{red}{\textbf{77.4{\tiny$\pm$0.6}}} & \textcolor{blue}{\textbf{77.0{\tiny$\pm$6.8}}} & 79.0{\tiny$\pm$3.3} & \textcolor{red}{\textbf{72.4{\tiny$\pm$5.7}}} & \textcolor{red}{\textbf{65.0{\tiny$\pm$2.0}}} & \textcolor{red}{\textbf{49.0{\tiny$\pm$2.2}}} & 33.8{\tiny$\pm$1.1} \\ \hline
\end{tabular}
\end{table*}

\begin{figure}[htbp]
    \centering
    \includegraphics[width=1\columnwidth]{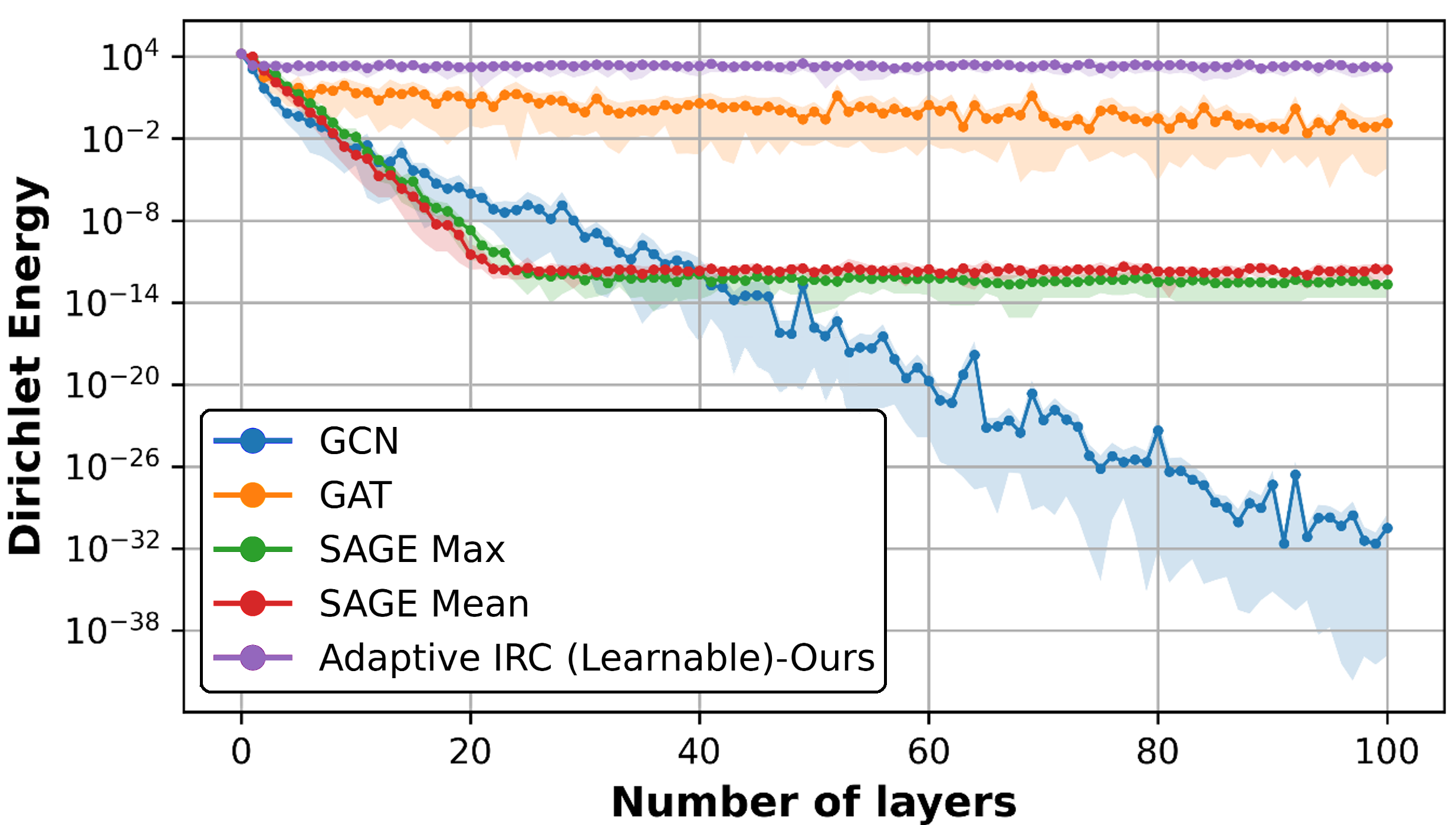}
    \caption{Dirichlet energy (log scale)   
    for output of GCN, GAT, GraphSAGE (mean and max), and learnable adaptive IRC with varying numbers of layers.}
    \label{dirichlet_energy}
\end{figure}

\begin{figure*}[htbp]
\includegraphics[width=2.2\columnwidth]{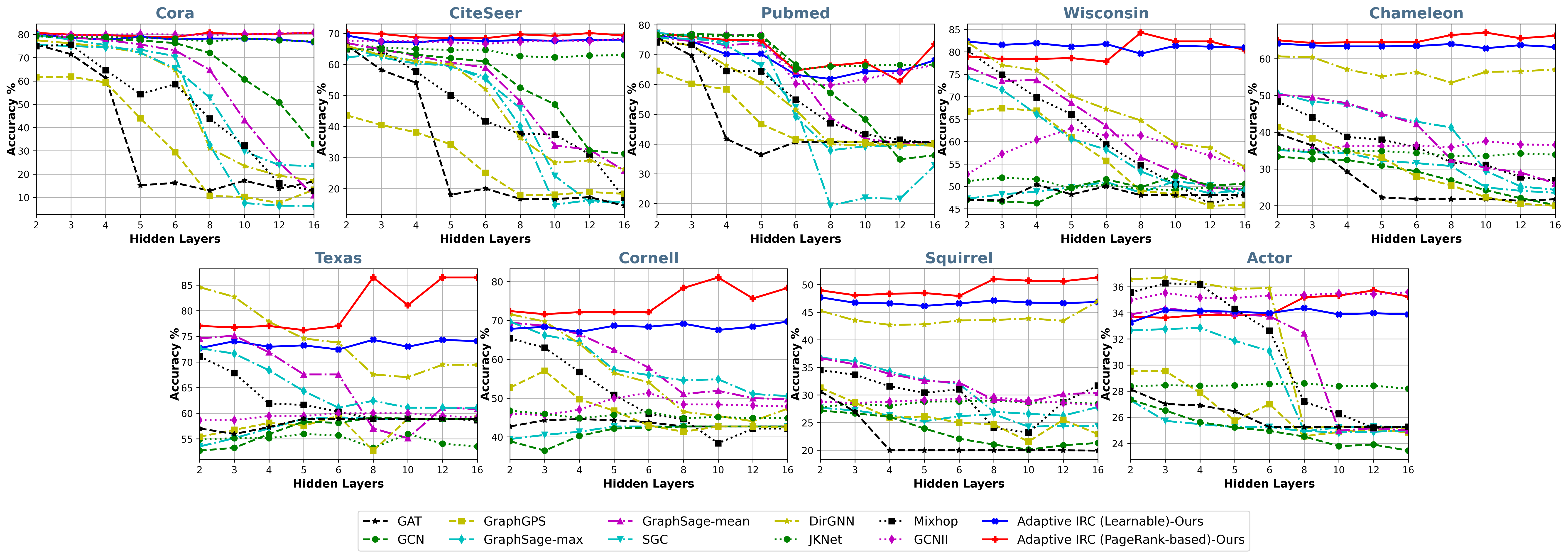}
    \caption{Performance across depths. Adaptive IRC (learnable and PageRank-based) remains accurate with increasing layers and outperforms other methods, except Actor, in shallow settings.}\label{depths}
\end{figure*}

\subsection{Node Classification}
We measure the performance of adaptive IRC on a range of common node classification benchmarks, using a collection of standard and state-of-the-art GNN models for comparison.

\paragraph{Other Methods.} 
We compare the adaptive IRC against some classical, but popular, GNN architectures like GCN~\cite{kipf2016semi}, GAT~\cite{velivckovic2017graph}, SGC~\cite{wu2019simplifying}, GraphSAGE~\cite{hamilton2017inductive_table}, Mixhop~\cite{abu2019mixhop}, JKNet~\cite{xu2018representation}, and GNCII~\cite{chen2020simple}. In addition, we compare our approach to some state-of-the-art GNN architectures, such as GraphGPS~\cite{rampavsek2022recipe} and DIRGNNs~\cite{rossi2024edge}. 

\paragraph{Our Methods.} 
We investigate two variants of the adaptive IRC model. First, the \textbf{learnable adaptive IRC}, where the diagonal residual strength matrix is optimized during training. While effective, it requires optimizing residual strengths for each node to achieve optimal results. A natural question is then: can we avoid learning these parameters and instead assign them heuristically to reduce time complexity while still benefiting from adaptivity? We observed (on small synthetic graphs) a positive correlation between a node's centrality, such as PageRank, and its residual strength; that is, the higher the centrality, the higher the residual strength. Based on this, we propose the \textbf{PageRank-based adaptive IRC}, where we compute PageRank scores for all nodes and assign $\lambda_{\max}$ to the top $k\%$ (with $k \in \{5, 7, 10\}$), and $\lambda_{\min}$ to the rest. \textit{This approach avoids learning residual strengths, reducing computational cost while continuing to perform comparably, and sometimes even better on average, than the learnable variant, as discussed below.}

\paragraph{Datasets. } 
To evaluate model performance across diverse scenarios, for the node classification task, we conduct experiments on homophilic graphs such as Cora~\cite{mccallum2000automating}, Citeseer~\cite{sen2008collective}, and Pubmed~\cite{namata2012query}, and on heterophilic graphs such as Texas, Wisconsin, and Cornell from WebK, as well as Chameleon, Squirrel, and Actor~\cite{rozemberczki2021multi,tang2009social}. Details of these datasets are summarized in Table~\ref{dataset_statistics_classification}. These datasets exhibit different degrees of homophily, quantified by the homophily ratio $H$ shown in Table~\ref{results_Heterophilic}. A higher ratio indicates greater homophily. For all datasets, we used $10$ random weight initializations. We followed the default train/validation/test splits for datasets provided by the \texttt{torch-geometric} library.

\begin{table}[htbp!]\small
\centering
\caption{Statistics of the datasets used in our experiments.}
\label{dataset_statistics_classification}
\renewcommand{\arraystretch}{1.2}
\begin{tabular}{|l|c|c|c|c|}
\hline
\textbf{Dataset} & \textbf{\#Nodes} & \textbf{\#Edges} & \textbf{\#Features} & \textbf{\#Classes} \\
\hline
Cora          & 2,708       & 10,556       & 1,433 & 7 \\
Citeseer      & 3,327       & 9,104        & 3,703 & 6 \\
PubMed        & 19,717      & 88,648       & 500   & 3 \\
Texas         & 183         & 309          & 1,703 & 5 \\
Wisconsin     & 251         & 499          & 1,703 & 5 \\
Cornell       & 183         & 295          & 1,703 & 5 \\
Chameleon     & 2,277       & 36,101       & 2,325 & 5 \\
Squirrel      & 5,201       & 217,073      & 2,089 & 5 \\
Actor         & 7,600       & 30,019       & 931   & 5 \\
\hline
\end{tabular}
\end{table}

\paragraph{Hyperparameters. }
For optimization, we employ the Adam optimizer with learning rates $\text{lr} \in \{10^{-2}, 10^{-3}\}$, weight decay set to $10^{-4}$, hidden dimensions selected from $\{64, 128\}$, the number of hidden layers from $\{2, 4\}$, a dropout rate of $0.4$, and residual strength assignments $\lambda_{\max} \in \{0.6, 0.7, 0.8\}$ and $\lambda_{\min} \in \{0.1, 0.2, 0.3\}$ for the PageRank-based version. For each model, the optimal hyperparameter configuration is determined via fine-tuning. 

\paragraph{Performance. }
Table~\ref{results_Heterophilic} summarizes the results of our experiments. Based on this table, except for the Actor dataset, the adaptive IRC consistently outperforms other methods. Between the two versions of our method, \textit{PageRank-based} performs as well as, and sometimes even better than, the learnable version. This is a surprising result, as one might expect the learnable version to achieve higher accuracy. Notably, there is a significant improvement in accuracy on heterophilic datasets compared to static residual connection models such as GCNII. To be specific, the PageRank-based IRC improves node classification accuracy (of the GCNII) by $17.5\%$, $18.6\%$, $25.4\%$, $28.8\%$, $20.2\%$, and $-1.4\%$ on the Texas, Wisconsin, Cornell, Chameleon, Squirrel, and Actor datasets, respectively, which shows the power of the adaptive IRC. 
This improvement stems from the adaptability of our method. Unlike GCNII, adaptive IRC adjusts the node-neighbor balance, preserving key node embeddings while leveraging neighbors when beneficial.

\paragraph{Depth. }
To demonstrate the adaptive IRC's versatility, we plot the accuracy of all models across all datasets as the number of hidden layers increases. As shown in Figure~\ref{depths}, while other methods degrade beyond four layers, both the learnable and PageRank-based variants of the adaptive IRC maintain stable and high performance even with deeper architectures. On the \textit{PubMed} dataset, accuracy drops from layer $5$ to $6$ across all models, likely due to the use of batch training in deeper networks. On the \textit{Actor} dataset, while some baselines perform better with shallow architectures (up to $5$ layers), our methods outperform them in deeper settings. The best hyperparameters were used for each model up to five layers. For models with more than five layers, we adopted the hyperparameters from the five-layer models.

\paragraph{The Best Hyperparameter of Adaptive IRC.}
The best hyperparameters from the experiments reported in Table~\ref{results_Heterophilic} are listed in Table~\ref{tab:hyperparams}. 
\begin{table}[ht]\small
\centering
\caption{The best hyperparameters for each dataset within the considered search space.}
\label{tab:hyperparams}
\begin{tabular}{|l|c|c|c|c|c|}
\hline
\textbf{Dataset} & lr  & d & k & \textbf{$\lambda_{\max}$} & \textbf{$\lambda_{\min}$} \\
\hline
\multicolumn{6}{|c|}{\textbf{Adaptive IRC (Learnable)}} \\
\hline
Actor      & 0.01  & 64  &  -- & -- & -- \\
Chameleon  & 0.001 & 128 &  -- & -- & -- \\
CiteSeer   & 0.001 & 64  &  -- & -- & -- \\
Cora       & 0.001  & 64 &  -- & -- & -- \\
Cornell    & 0.01  & 64  &  -- & -- & -- \\
Pubmed     & 0.01  & 128 &  -- & -- & -- \\
Squirrel   & 0.001 & 128 &  -- & -- & -- \\
Texas      & 0.01  & 128 &  -- & -- & -- \\
Wisconsin  & 0.01  & 64  &  -- & -- & -- \\
\hline
\multicolumn{6}{|c|}{\textbf{Adaptive IRC (PageRank-based)}} \\
\hline
Actor      & 0.01  & 64  & 0.07 & 0.6 & 0.1 \\
Chameleon  & 0.001  & 128 & 0.05 & 0.6 & 0.3 \\
CiteSeer   & 0.001 & 128  & 0.05 & 0.6 & 0.3 \\
Cora       & 0.001 & 128 & 0.10 & 0.7 & 0.3 \\
Cornell    & 0.01  & 64 & 0.07 & 0.8 & 0.1 \\
Pubmed     & 0.01  & 64  & 0.07 & 0.6 & 0.3 \\
Squirrel   & 0.001 & 128 & 0.05 & 0.6 & 0.3 \\
Texas      & 0.01  & 64 & 0.10 & 0.6 & 0.2 \\
Wisconsin  & 0.01  & 64 & 0.05 & 0.6 & 0.1 \\
\hline
\end{tabular}
\end{table}

\paragraph{Number of Parameters.}
Under our default 4-layer configuration on Cora, for example, the number of learnable parameters varies significantly across models.  
Simpler architectures, such as \textbf{GCN}, \textbf{SGC}, \textbf{JKNet}, and \textbf{GCNII}, each have around \textbf{100K} parameters, while \textbf{GAT} has substantially more at \textbf{840K}. \textbf{GraphSAGE-mean}, \textbf{GraphSAGE-max}, \textbf{GraphGPS}, and \textbf{MixHop} use more than \textbf{200K},  and \textbf{DirGNN} \textbf{140K}. Our \textbf{Adaptive IRC} models have higher counts: \textbf{483K} (Learnable) and \textbf{481K} (PageRank-based). Although the parameter gap between the two variants is small in these benchmarks (with a maximum number of nodes less than $19{,}717$), real-world graphs are often orders of magnitude larger. In such cases, the PageRank-based model avoids learning $n$ node-specific parameters, resulting in significant savings. 

\section{Conclusion and Future Work}
We study an adaptive residual scheme in which different nodes can have varying residual strengths. Our analysis demonstrates that this prevents oversmoothing (the Dirichlet energy remains non-zero) and maintains the embedding matrix's rank across layers. To improve the time complexity of our approach, we introduce a variant in which residual strengths are not learned but are set heuristically, a choice that performs as well as the learnable version. An interesting future direction is the \textit{adaptive selection of residual strengths} without learning, as in our PageRank-based approach, which outperforms the learnable variant and other residual-based GNNs. A practical strategy is to assign shared residual strengths to groups of structurally similar nodes, leading to more effective and interpretable designs.

\clearpage



\vspace{-1cm}

\end{document}